\documentclass[a4paper,10pt]{article}
\usepackage{nips11submit_e}
\usepackage{graphicx}
\usepackage{tikz}
\usetikzlibrary{arrows}
\usepackage{pgfplots}
\usepackage{subfig}
\usepackage{floatrow}
\usepackage{caption}
\usepackage[ruled,vlined]{algorithm2e}

\nipsfinalcopy
\usepackage{ifthen}
\usepackage{amsmath, amsfonts, amsthm}

\newboolean{showTodos}
\newboolean{showRemarks}
\newboolean{showExclude}

\setboolean{showTodos}{true}
\setboolean{showRemarks}{true}
\setboolean{showExclude}{false}

\newcommand{\todo}[2][]{\ifthenelse{\boolean{showTodos}}{\optionalTextFormat{Todo}{#1}{#2}}{}}
\newcommand{\remark}[2][]{\ifthenelse{\boolean{showRemarks}}{\optionalTextFormat{Remark}{#1}{#2}}{}}
\newcommand{\exclude}[2][]{\ifthenelse{\boolean{showExclude}}{\optionalTextFormat{Exclude}{#1}{#2}}{}}

\newcommand{\alternative}[3]{
  \ifthenelse{\boolean{#1}}{
    [\textit{Start of Alternative}] #2
    
    \optionalTextFormat{Alternative Text}{#3}
  }{#2}
}

\newcommand{\optionalTextFormat}[3]{[\textbf{#1}#2{:~} \textit{#3}]}

\newcommand{\refeq}[1]{(\ref{#1})}

\theoremstyle{plain}
\newtheorem{theorem}{Theorem}
\newtheorem{lemma}[theorem]{Lemma}

\theoremstyle{definition}

\newcommand{\DeclareMathVarIndex}[2]{\DeclareMathOperator{#1}{#2}}

\renewcommand{\vec}[1]{\mathbf{#1}}
\newcommand{\mat}[1]{\mathbf{#1}}
\newcommand{\fun}[1]{\mathcal{#1}}
\newcommand{\set}[1]{\mathcal{#1}}
\newcommand{\op}[1]{\ensuremath{\mathcal{#1}}}

\newcommand{\DeclVec}[3]{\ensuremath{\vec{#1} \in #2^{#3}}}

\DeclareMathVarIndex{\new}{new}
\DeclareMathVarIndex{\old}{old}
\DeclareMathVarIndex{\opt}{*}

\newcommand{\Exp}[2]{\mathbb{E}_{#1}\left[#2\right]}
 
\newcommand{\dd}{\text{d}}

\newcommand{\grad}{\nabla}

\newcommand{\trans}[1]{#1^{T}}
\newcommand{\inv}[1]{#1^{-1}}

\newcommand{\norm}[2][]{\|#2\|_{#1}}
\DeclareMathVarIndex{\MAP}{MAP}
\newcommand{\Real}{\mathbb{R}}

\newcommand{\x}{\ensuremath{\vec{x}}}
\renewcommand{\u}{\ensuremath{\vec{u}}}

\newcommand{\F}{\ensuremath{f}}
\newcommand{\B}{\ensuremath{\mat{B}}}
\newcommand{\Q}{\ensuremath{\mat{Q}}}

\renewcommand{\H}{\ensuremath{\mat{H}}}

\newcommand{\Copt}{\ensuremath{\C^{\opt}}}
\newcommand{\C}{\ensuremath{J}}  
\newcommand{\Ct}{\ensuremath{C}} 
\newcommand{\CT}{\ensuremath{C_{\bullet}}} 
\newcommand{\Cskill}{\ensuremath{\Ct_{skill}}}
\newcommand{\Ctask}{\ensuremath{\Ct_{task}}}
\newcommand{\expC}{\ensuremath{\Phi}}

\newcommand{\nullpi}{\ensuremath{{\pi^{0}}}}

\newcommand{\bPsi}{\ensuremath{\bar{\Psi}}}

\newcommand{\optpi}{\ensuremath{\pi^*}}

\newcommand{\path}[3]{\ensuremath{#1(#2\rightarrow#3)}}

\newcommand{\kernel}[1]{\ensuremath{#1}}
\newcommand{\rkhs}[2]{\ensuremath{\mathcal{#1}^{\kernel{#2}}}}
\newcommand{\gram}[2]{\vec{g}^{\kernel{#1}}_{\set{#2}}}
\newcommand{\Gram}[2]{\vec{G}^{\kernel{#1}}_{\set{#2}}}
\newcommand{\ka}{\kernel{k}}
\newcommand{\kb}{\kernel{l}}
\newcommand{\kc}{\kernel{w}}
\newcommand{\Ha}{\rkhs{H}{k}}
\newcommand{\Hb}{\rkhs{H}{l}}
\newcommand{\Hc}{\rkhs{H}{w}}

\newcommand{\kpsi}{\psi}
\newcommand{\kphi}{\phi}
\newcommand{\HPsi}{\rkhs{H}{\psi}}
\newcommand{\HePsi}{\rkhs{H}{\psi'}}
\newcommand{\HPhi}{\rkhs{H}{\phi}}

\newcommand{\Xu}{\ensuremath{X^\pi}}
\newcommand{\Xz}{\ensuremath{X^0}}
\newcommand{\Xnu}{\ensuremath{X^\nu}}
\newcommand{\bX}{\ensuremath{X}}

\newcommand{\opE}[1]{\op{E}^{#1}}
\newcommand{\opU}[1]{\op{U}^{#1}}
\newcommand{\opC}[1]{\op{C}^{#1}}
\newcommand{\OpE}[2]{\op{E}^{#1}\left[#2\right]}
\newcommand{\OpU}[2]{\op{U}^{#1}\left[#2\right]}
\newcommand{\OpR}[2]{\op{R}_{#1}\left[#2\right]}

\newcommand{\ophU}[2]{\hat{\op{U}}^{#1}_{\set{#2}}}
\newcommand{\ophC}[2]{\hat{\op{C}}^{#1}_{#2}}

\newcommand{\OphU}[3]{\hat{\op{U}}^{#1}_{\set{#2}}\left[#3\right]}

\newcommand{\W}{\mat{W}}

\newcommand{\dimX}{\ensuremath{d_x}}
\newcommand{\dimU}{\ensuremath{d_u}}

\newcommand{\z}{\ensuremath{z}}
\newcommand{\y}{\ensuremath{y}}
\newcommand{\Z}{\ensuremath{Z}}
\newcommand{\Y}{\ensuremath{Y}}
\newcommand{\setZ}{\ensuremath{\set{Z}}}
\newcommand{\setY}{\ensuremath{\set{Y}}}

\newcommand{\tX}{\ensuremath{\tilde{X}}}

\usepackage{amsthm}

\title{Path Integral Control by Reproducing Kernel Hilbert Space Embedding}
\author{
Konrad Rawlik \\
School of Informatics\\
University of Edinburgh\\
Edinburgh, UK\\
\And
Marc Toussaint\\
FU Berlin\\
Berlin, Germany\\
\And
Sethu Vijayakumar\\
School of Informatics\\
University of Edinburgh\\
Edinburgh, UK\\
}

\begin{document}
\maketitle
\begin{abstract}
We present an embedding of stochastic optimal control problems, of the so called path integral form, into reproducing kernel Hilbert spaces. Using consistent, sample based estimates of the embedding leads to a model free, non-parametric approach for calculation of an approximate solution to the control problem. This formulation admits a decomposition of the problem into an invariant and task dependent component. Consequently, we make much more efficient use of the sample data compared to previous sample based approaches in this domain, e.g., by allowing sample re-use across tasks. Numerical examples on test problems, which illustrate the sample efficiency, are provided.
\end{abstract}

\section{Introduction}
While solving general non-linear stochastic optimal control and Reinforcement Learning problems remains challenging, some recent work \cite{Kappen:PI} has identified a class of problems that admit closed form solutions. Although these solutions require evaluation of a path integral -- equivalent to evaluation of a partition function, which in itself is a hard problem -- they allow for the application of Monte Carlo and Variational methods, leading to  several practical applications, e.g., \cite{Theodorou:PIPI, Kappen:Constrained}. In the special case of linear dynamics and quadratic costs, the required path integral can be evaluated analytically based on linear operators acting on state vectors. Here, we show that, analogously, a suitable embedding of the path integral into a \emph{reproducing kernel Hilbert space} (RKHS) allows it's evaluation in terms of covariance operators acting on elements of the Hilbert space. While this in itself does not yield a tractable solution to the SOC problem, consistent estimators of the required operators give rise to efficient non-parametric algorithms.

The change of perspective from the direct estimation of the path integral (which previous applications of Monte Carlo methods aimed at) to estimation of operators allows to overcome several shortcomings of previous methods while maintaining many of their advantages. Most importantly, it can significantly reduce the sample complexity by splitting the problem appropriately into an invariant and task varying component, allowing efficient sample re-use across tasks and leading to a form of transfer learning -- contrast this to the situation where any change in the task including, for e.g., different start states, necessitate acquiring new samples \cite{Theodorou:PIPI,Theodorou:PI:RBDyn}. Additionally, the approach remains model free, allowing it's application to the Reinforcement Learning setting. This is in contrast to variational \cite{Kappen:EP} or function approximation \cite{Todorov:Aprox:Agr, Todorov:Aprox:LS} approaches, from which it is further distinguished through convergence guarantees. The RKHS embedding make the operators state-dimensionality independent, leading to better scalability, while prior knowledge about both tasks and dynamics can be effectively incorporated by informing choices of sampling procedures and kernel.

It is worth noting that, while we choose to frame our approach in the context of path integral stochastic optimal control, it is not restricted to problems which fall into this class. The formalisms of linearly solvable MDPs \cite{Todorov:MDP}, inference control \cite{Toussaint:AICO} and free energy control \cite{Friston} all require solving an underlying problem of equivalent form, making the methods proposed directly applicable in these contexts. Furthermore \cite{Rawlik:RSS} discusses a formulation which generalizes path integral control, to derive an optimal policy for general SOC problems. Finally, while we focus on finite horizon problems, path integral formulations for discounted and average cost infinite horizon problems \cite{Todorov:PNAS}, as well as risk sensitive control \cite{Kappen:RiskPI} also exist.

\section{Path Integral Control}
In this section we briefly review the path integral approach to stochastic optimal control \cite{Kappen:PI}, for a more detailed treatment, see \cite{Kappen:PI:Chapter, Theodorou:PIPI}.
Let $\DeclVec{\x}{\Real}{\dimX}$ be the system state and $\DeclVec{\u}{\Real}{\dimU}$ the control signals. Consider a continuous time stochastic system of the form
\begin{equation}
\label{eq:dyn}
 \dd\x = \F(\x,t)\dd t + \B(\x,t)(\u \dd t + \dd\xi)~,
\end{equation}
where $\dd\xi$ is a multivariate Wiener process with $\Exp{}{d\xi^2} = \Q(\x,t)dt$, and $\F$, $\B$ and $\Q$ may be non-linear functions. In particular note that the system is affine in the controls and both noise and controls act in the same subspace. We seek the best Markov policy, i.e., $\u(t) = \pi(\x(t),t)$, with respect to an objective of the form
\begin{equation}
 \label{eq:cost}
   \C^\pi(\x, t) = \Exp{\path{\Xu}{t}{T}|\x}{\CT(\Xu(T)) + \int_{t}^T \Ct(\Xu(s), s) + \trans{\u(s)}\H \u(s) \dd s}~,
\end{equation}
where $T$ is some given terminal time and the expectation is taken w.r.t. to path of \refeq{eq:dyn} starting in $\x$ and following policy $\pi$. The control cost is further constrained by requiring it to satisfy $\mat{Q} = \lambda\mat{B}\inv{\H}\trans{\mat{B}}$ for some constant scalar $\lambda > 0$.

It can be shown that for problems of this form the optimised objective can be expressed as
\begin{equation}
  \label{eq:Copt}
  \Copt(\x,t) = \min_\pi \C^\pi(\x,t) = -\lambda\log\Psi(\x,t)~,
\end{equation}
where $\Psi$ is given by the path integral
\begin{equation}
 \label{eq:pi}
 \Psi(\x,t) = \Exp{\path{\Xz}{t}{T}|\x}{e^{-\int_{t}^T \frac{1}{\lambda}\Ct(\Xz(s),s)\dd s}\Psi(\Xz(T),T)}~,
\end{equation}
with $\Psi(\cdot,T) = \exp\{-\CT(\cdot)/\lambda\}$. The expectation in \refeq{eq:pi} is taken w.r.t. uncontrolled path of the dynamics \refeq{eq:dyn}, i.e. those under the policy $\nullpi(\cdot,\cdot) = 0$, starting in $\x_t$.

As a consequence of linear control with quadratic control cost and \refeq{eq:Copt}, the optimal policy $\optpi(\x,t)$ can be expressed directly in terms of $\Psi$ as
\begin{equation}
  \label{eq:optpi:cont}
  \optpi(\x,t) = -\inv{\H}\trans{\B(\x)}\grad_\x\Copt(\x,t) = \inv{\H}\trans{\B(\x)}\frac{\lambda \grad_\x \Psi(\x,t)}{\Psi(\x,t)}~,
\end{equation}
making obtaining $\Psi$ the main computational challenge for problems in this class.

Assuming we are only interested in the optimal controls at certain time points, say $\{t_{1, \dots, n}\}$ with $t_n = T$, it is sufficient to compute the set $\Psi_i(x) = \Psi(x,t_i)$ and  \refeq{eq:pi} admits a representation in terms of the finite dimensional distribution $\bX = (\Xz(t_0),\cdots, \Xz(t_n))$. Specifically using the Markov property of $\Xz(t)$ and marginalising intermediate states we obtain the recursive expression
\begin{align}
\label{eq:pi:finite}
\Psi_i(x_{t_i})
    & = \Exp{\bX_{i+1}|x_{t_i}}{\expC_{i}(x_{t_i},\bX_{i+1})\cdot\Psi_{i+1}(\bX_{i+1})}~.
\end{align}
Here,
\begin{equation}
  \label{eq:problem:local}
  \expC_i(x_{t_i},x_{t_{i+1}}) = \Exp{\path{\Xz}{t_i}{t_{i+1}}|x_{t_i},x_{t_{i+1}}}{e^{-\frac{1}{\lambda}\int_{t_i}^{t_{i+1}}\Ct(\Xz(s),s)\dd s}}~,
\end{equation}
where the expectation is taken w.r.t. uncontrolled path from $x_{t_i}$ to $x_{t_{i+1}}$. Note that $-\lambda\log\expC_i$ can be seen as the (optimal) expected cost for the problem of going from $x_{t_i}$ to $x_{t_{i+1}}$ over the time horizon $[t_i,t_{i+1}]$ under dynamics and running costs corresponding to those of the overall problem given in \refeq{eq:cost}. Hence, the problem naturally decomposes into, on the one hand, a set of short horizon problems -- or indeed a nested hierarchy of such $\expC$ -- and on the other hand, a set of recursive evaluations backwards in time.

\section{Embedding of the Path Integral}
We now demonstrate that \refeq{eq:pi:finite} can be expressed in terms of linear operators in RKHSs. While the exposition is necessarily short, \cite{RKHS} provides a more through treatment of the theory of RKHSs while \cite{Smola07, Song:kFilter, Gretton:KBP, Gretton:KBR} provide the basic concepts on which we build.

\subsection{Analytical One Step Path Integral Embedding}
\label{sec:emb:analytical}
Let $\Ha$ denote the reproducing kernel Hilbert space of functions $\setZ \rightarrow \Real$ associated with the positive semi-definite kernel $\ka(\cdot,\cdot)$.
Further, let $\set{P}^{\setZ}$ be the set of random variables on $\setZ$. Following \cite{Smola07}, we define the embedding operator $\opE{\ka}:\set{P}^{\setZ}\rightarrow\Ha$ by
\begin{equation}
 \label{eq:op:E}
 	\left<h, \OpE{\ka}{\Z}\right> = \Exp{\Z}{h(\Z)}  \quad \forall \Z\in\set{P}^{\setZ}, h\in\Ha~,
\end{equation}
which constitutes a direct extension of the standard embedding of individual elements $\z\in\setZ$ into $\Ha$ used more commonly in the literature.

In the problem under consideration, the interest lies with the evaluation of $\Psi_i$ given in \refeq{eq:pi:finite} and hence, in a suitable embedding of $\bX_{i+1}|x_i$ which would allow the required expectation to be expressed as an inner product in some RKHS. Although \refeq{eq:op:E} can be directly applied -- since for fixed $x_i$, $\bX_{i+1}|x_i$ is a simple random variable -- it is convenient to consider a general conditional random variable $\Z|\y$ as a map $\setY \rightarrow \set{P}^{\setZ}$, yielding random variables over $\setZ$ given a value $\y \in \setY$, and define a conditional embedding operator $\opU{\kb\ka}: \Hb \rightarrow \Ha$ s.t.
\begin{equation}
 \label{eq:op:U}
 \OpE{\ka}{\Z|\y} = \opU{\kb\ka}\circ\OpE{\kb}{\y}~,
\end{equation}
where $\OpE{\kb}{\y} = \kb(\cdot, \y)$, i.e., the standard embedding operator of elements $\y\in\set{Y}$ used in kernel methods. An explicit form of the operator $\opU{}$ is given in \cite{Song:kFilter} by means of \emph{covariance operators}, which are generalizations of covariance matrices. Specifically, the uncentered covariance operator $\opC{\ka\kb}_{\Z\Y}$ for the joint random variable $(\Z,\Y)$ is given by
\begin{equation}
   \label{eq:op:Cov}
   \opC{\ka\kb}_{\Z\Y} = \Exp{(\Z,\Y)}{\ka(\Z,\cdot)\otimes\kb(\Y,\cdot)}~,
\end{equation}
where $\otimes$ denotes the tensor product. Note that we can see $\opC{\ka\kb}_{\Z\Y}$ as an embedding of $(\Z,\Y)$ into the tensor product space $\Hc = \Ha\otimes\Hb$, which is the RKHS of the product kernel $\kc((\z,\y),(\z',\y')) = \ka(\z,\z')\kb(\y,\y')$. Now, under certain technical considerations detailed in \cite{Gretton:KBR} but beyond the scope of this paper,
\begin{equation}
  \label{eq:op:U:asCov}
    \opU{\kb\ka} = \opC{\ka\kb}_{\Z\Y}\inv{\left(\opC{\kb\kb}_{\Y\Y}\right)}
\end{equation}
satisfies \refeq{eq:op:U}.

However, as the argument of the expectation, specifically of $\expC$, is not only a function of the random variable, i.e., $\bX_{i+1}$, but also of the conditioning $x_i$, we can not apply \refeq{eq:op:U} directly. We proceed by introducing an auxiliary random variable $\tX$ such that $P(\tX,\bX_{i+1}|\x_i) = P(\bX_{i+1}|\x_i)\delta_{\tX=\x_i}$ with $\delta$ the delta distribution, hence
\begin{equation}
      \left<h, \OpE{\ka}{\bX_{i+1},\tX|\x_i}\right> = \Exp{\bX_{i+1},\tX|\x_i}{h(\tX,\bX_{i+1})} = \Exp{\bX_{i+1}|\x_i}{h(\x_i,\bX_{i+1})}\quad\forall h\in\Ha~.
\end{equation}
Note that treating $\x_i$ as constant leads to an alternative formulation. This, although equivalent in the analytical setting, does however not immediately yield a practical empirical estimator as further discussed in the supplementary material.

Now, assume $\HPsi,\HPhi$, s.t. $\Psi \in \HPsi$, $\expC \in \HPhi$, are given\footnote{n.b., $\HPhi$ is a space of functions $\Real^{\dimX} \times \Real^{\dimX} \rightarrow \Real$, while $\HPsi$ contains functions $\Real^{\dimX} \rightarrow \Real$}. To account for the mismatch in the arity of functions in these spaces, $\HPsi$ may be trivially extended to $\HePsi$, a space of functions $\Real^{\dimX}\times\Real^{\dimX} \rightarrow \Real$, using the kernel $\psi'((u,v),(u',v')) = \psi(u,u')$, i.e., we identify $\HPsi$ and its tensor product with the RKHS of constant functions. Hence, taking the embedding of $\bX_{i+1},\tX|x_i$ into $\Hc = \HPhi \otimes \HePsi$ in which the product function of $\expC_i$, $\Psi_{i+1}$ resides, using \refeq{eq:pi:finite} and further applying \refeq{eq:op:E} and \refeq{eq:op:U} we have
\begin{align}
\Psi_i(\x)
& = \Exp{\bX_{i+1}|\bX_{i}=\x}{\expC_i(\x, \bX_{i+1})\cdot\Psi_{i+1}(\bX_{i+1})} \\
& = \left<\expC_i\otimes\Psi_{i+1}, \OpE{\kc}{\bX_{i+1},\tX|\bX_{i} = \x}\right> \\
\label{eq:pi:rkhs:analytical:b}
& = \left<\expC_i\otimes\Psi_{i+1}, \opU{\kc\ka}\circ\OpE{\ka}{\x}\right>~,
\end{align}
where $\ka$ is some kernel over $\Real^{\dimX}$ of our choosing. As will become apparent in the following (see \refeq{eq:pi:rkhs:emp}), it is convenient for computational reasons to take $\ka$ to be $\psi$ as it allows for re-use of pre-computed matrices over the recursive evaluation of estimates of $\Psi$.


\subsection{Finite Sample Estimates}
\label{sec:emb:finitesample}
Evaluation of $\op{U}$ -- thus also of the path integral embedding \refeq{eq:pi:rkhs:analytical:b} -- requires evaluation of expectations of kernels and remains therefore, in most cases, intractable. However as the operators are expressed in terms of expectations, it is straightforward to form empirical estimates, leading to practical algorithms.

First consider the general case, given a set $\set{D} = \{(\z,\y)_{0 \dots m}\}$ of i.i.d. samples from $(\Z,\Y)$. An estimate of \refeq{eq:op:Cov} is given by
\begin{equation}
\label{eq:emp:opE_opC}
\ophC{\ka\kb}{\set{D}} = \frac{1}{m}\sum_{i=1}^m \ka(\cdot, \z_i)\otimes\kb(\cdot,\y_i)~.
\end{equation}
Using the latter in conjunction with \refeq{eq:op:U:asCov}, a regularized estimate of $\opU{\kb\ka}$ is given by
\begin{equation}
\label{eq:op:U:emp}
\ophU{\kb\ka}{D} = \vec{g}^{\ka}_{\setZ}\inv{(\mat{G}^{\kb}_{\setY\setY} + \epsilon m\mat{I})}\gram{\kb}{\Y}~,
\end{equation}
where $\epsilon$ represents a regularization parameter and $\gram{\ka}{A}$, $\Gram{\ka}{AB}$ represents the vector of embeddings and Gramian respectively, i.e. $[\gram{\ka}{A}]_i = \ka(a_i,\cdot)$ and $[\Gram{\ka}{(A,B)}]_{ij} = \ka(a_i,b_j)$, for given sets $\set{A},\set{B}$ and kernel $\ka$.

Now, turning to the specific expression of interest, i.e., $\Psi$ in \refeq{eq:pi:rkhs:analytical:b}, we can form an empirical estimate based on $\set{D} = \{(x,x')_{1 \dots m}\}$ sampled i.i.d. from a joint distribution $P(X',X) = p_{\nullpi}(X'|X)\mu(X)$, s.t. $p_{\nullpi}(X'|X)$ is the p.d.f. of $\bX_{i+1}|\bX_i$ and $\mu$ is a free prior. Specifically, assume the representation of $\expC_i$ in $\HPhi$ is $\gram{\kphi}{B}\beta$, which we do not assume to be finite dimensional.  Then, given a empirical estimate $\bPsi_{i+1} = \gram{\kpsi}{A}\alpha_{i+1}$, based on some set $\set{A}$, we obtain the estimate
\begin{equation}
\label{eq:pi:rkhs:emp}
\bPsi_{i} = \gram{\kpsi}{X}\alpha_i \quad\quad\text{with}\quad\quad \alpha_{i} = \trans{\left[\Gram{\kphi}{DB}\beta \odot \Gram{\kpsi}{X'A}\alpha_{i+1}\right]}\inv{(\Gram{\kpsi}{XX} + \epsilon m \mat{I})}~,
\end{equation}
where $\odot$ denotes the Hadamard product. The term $\Gram{\kphi}{DB}\beta$ takes-- assuming without loss of generality, $\expC \in \HPhi$-- the particularly simple form
\begin{equation}
\label{eq:gram:expC}
\Gram{\kphi}{DB}\beta =  \expC(\set{X},\set{X}')  =  \trans{(\expC(x_1,x'_1), \expC(x_1,x'_2),  \dots)}~.
\end{equation}
Hence, obtaining an explicit representation of $\expC$, or indeed choosing $\HPhi$, is not necessary.

Importantly, note that $\bPsi_i$ is a finite weighted sum of kernels, hence, $\bPsi_i \in \HPsi$, which directly allows a recursive computation of all $\bPsi_1 \dots n$ and leads, using \refeq{eq:optpi:cont}, to an approximate optimal policy for fine discretisations of the problem. Furthermore, all required matrices are functions of the sample data only and as such can be pre-computed. Finally, the estimator is consistent (see supplementary material for proof). While for bounded expected costs, convergence of $\bPsi$ implies convergence of the estimate of the expected cost (see supplementary material), convergence of the latter can be slow for large values due to the log transform, leading in practice to poor policies in regions where $\Psi$ is small. We would like to emphasize that this problem is not limited to the methods proposed here, but is a characteristic of any approach based on estimation of $\Psi$, e.g., as also noted by \cite{Todorov:Aprox:LS}. To overcome this problem in practice, we form a Laplace approximation to $\bPsi$ at a local mode and use it where $\bPsi$ is small - this corresponds to a local quadratic approximation of the value function,  resulting in a linear policy which steers the system towards regions of high $\bPsi$.

\section{Efficient Estimators}
\label{sec:estimators}
The basic estimator \refeq{eq:pi:rkhs:emp} has several drawbacks. For one it has a relatively high computational complexity of $\fun{O}(m^3)$ for the matrix inversion, only required once if the same $\set{D}$ is used in each time step, and subsequently $\fun{O}(m^2)$ per iteration. Additionally, sample data under the uncontrolled dynamics is required, thus not allowing for off-policy learning. To overcome these problems two alternative estimators based on weighted samples, which partly address these issues, are discussed in the supplementary material. Specifically, the estimator employs Gram-Schmidt orthogonalisation, presented previously by \cite{Gretton:KBP}, which reduces the computational complexity to $\fun{O}(\hat{m}^2)$ with $\fun{O}(\hat{m}^3 + \hat{m}^2m)$ pre-computations for a chosen $\hat{m} \ll m$, and a novel importance sampling based estimator. We choose to defer the discussion of these in order to address a, in our opinion,   often overlooked aspect of efficiency when solving varying problems under the same dynamics. In practice, tasks are not performed in isolation, rather varying instances of often related problems have to be solved repeatedly, e.g., an optimized single reaching movement is of limited use since complex interactions require a series of such movements with changing start and target states. Previous approaches generally assume re-initialisation for each problem instance, e.g., Monte Carlo methods require novel samples, even under such trivial changes as the start state. In the following, we discuss extensions to the proposed method which improve sampling efficiency in exactly these cases, allowing efficient sample re-use over repeated applications.

\subsection{Transfer Learning via Transition Sample Re-use}
\label{sec:dyn:sample_reuse}
A limitation of the estimator arising in practice is the necessity of evaluating $\expC$ at the training transitions (cf. \refeq{eq:pi:rkhs:emp} and \refeq{eq:gram:expC}) which, in general, may be infeasible. It is therefore desirable to obtain an estimator based on evaluation of $\expC$ on a separate, ideally arbitrary, data set $\set{D'}$. Observe that
\begin{equation*}
	\Gram{\kphi}{DB}\beta
	= \langle\expC,\kphi(\set{D},\cdot)\rangle
	= \langle\expC,\opC{\kphi\kphi}_{ZZ}\inv{\left(\opC{\kphi\kphi}_{ZZ}\right)}\kphi(\set{D},\cdot)\rangle
	\approx \overbrace{\trans{\beta}\Gram{\kphi}{BD'}}^{\expC(\set{D}')} \inv{(\Gram{\kphi}{D'D'} + \epsilon m' \mat{I})}\Gram{\kphi}{D'D}~,
\end{equation*}
where $Z$ is an some free random variable with support on $\Real^{\dimX}\times\Real^{\dimX}$ and we used an empirical estimator based on a data set $\set{D'} = \{(x,x')_{1\dots m'}\}$ of i.i.d. samples from $Z$ (often in practice $\set{D'} \subseteq \set{D}$). As indicated evaluation of the r.h.s. only requires evaluation of $\expC$ at elements of $\set{D}'$, hence substituting into \refeq{eq:pi:rkhs:emp} gives the desired result. In particular we are now able to pre-compute and re-use the inverse matrix of \refeq{eq:pi:rkhs:emp} across changing tasks and, assuming time stationary dynamics, across different time steps. This is of importance for efficient estimation in, e.g., the Reinforcement Learning setting where incurred costs are known only at observed transitions or in cases where $\expC$ can be freely evaluated but it is expensive to do so, while generating large sets of transition samples may be comparatively cheap, e.g., the case of simple kinematic control where cost evaluation requires collision detection. Note that this form makes explicit use of the kernel $\kphi$, and while we may not be able to guarantee $\expC \in \rkhs{H}{\kphi}$, by choosing a kernel such that the projection of $\expC$ onto $\rkhs{H}{\kphi}$ is close to $\expC$, we can expect good results.

\subsection{Task augmented sampling}
\label{sec:abstraction:task}
We now turn to the question of the sampling distribution. While in general samples are required from the task agnostic dynamics $\Xz$, a task often induces regularities which suggests more suitable sampling distributions. In particular considering the role $\expC$ takes in \refeq{eq:pi:rkhs:emp} as a weight vector, it appears desirable, akin to importance sampling, to concentrate samples in regions of high $\expC$. Obviously $\expC$ can be used to guide the choice of the prior $\mu$ (c.f. Section \ref{sec:emb:finitesample}), however, in the context of repeated tasks we can go further and incorporate $\expC$ partly into the sampling process allowing, amongst others, for incremental learning of the task.



Consider the specific situation where one wishes to execute several task instances of a generic skill. This situation is often characterised by an invariant cost component relating to the skill and a task specific cost component -- if one looks at walking as an example, we wish to stay balanced in each step but the foot placement target will differ from step to step. Formally assume the state cost decomposes as
\begin{equation}
\label{eq:cost:factored}
    \Ct(x,\theta,t) = \Cskill(\x,t) + \Ctask(\x,\theta,t) ~,
\end{equation}
where $\theta$ parameterises the task. In this case, we may write the path integral \refeq{eq:pi} as
\begin{equation}
  \label{eq:pi:task}
    \Psi =  \Exp{\path{\Xnu}{t}{T}|x_t}{e^{-\int_t^T \frac{1}{\lambda}\Ctask(\Xnu(t),\theta,t)}\Psi(\Xnu(T),T)}~,
\end{equation}
where the expectation is now taken w.r.t. path of $\Xnu$, which are modified dynamics which absorb the invariant skill component of the cost, i.e., they bias the path dynamics based on the Fokker-Plank equation
\begin{equation}
  \label{eq:dyn:aug}
    \partial_t \nu = - \frac{\Cskill}{\lambda}\nu - \grad_x \left(f\nu\right) + \grad_x^2 \left(\mat{B}\trans{\mat{B}}\nu\right)~,
\end{equation}
in other words, the augmented dynamics tends to restrict the solutions to lie on, or at least stay close to, some \emph{skill space}.

A practical approach for exploiting the induced structure, is to learn the relevant subspace from a few example demonstrations sampled, using e.g. the approach in \cite{Havoutis:ICRA}, and sample $\set{D}$ on the learned space. Such explicit learning of the space has several advantages; foremost, we can use knowledge of the space to choose an appropriate kernel. Also, while $\Ctask$ is generally well defined by specific objectives we wish to achieve, the skill component often takes a more abstract form, e.g. we may desire movements to overall appear 'natural', and may only be given implicitly by expert demonstrations of desired movements, in which case the proposed framework allows \refeq{eq:pi:task} to be used to perform optimal control without explicitly referring to the implicit costs.

\section{Experimental Validation}

\subsection{Double Slit}

We first consider the double slit problem, previously studied by \cite{Kappen:PI} to demonstrate Monte Carlo approaches to path integral control. The problem is sufficiently simple to allow for a closed form solution for $\Psi$ to be obtained, but complex enough to highlight the shortcomings of some previous approaches. The task concerns a particle moving with constant velocity in one coordinate, while noise and controls affects it's position in an orthogonal direction. The aim is to minimise the square error to a target position at some final time, while also avoiding obstacles at some intermediate time, as illustrated in Fig.~\ref{fig:ds}(a). Specifically, the one dimensional dynamics are $\dd x = u + \dd\xi$ and the cost is given by
\begin{equation}
   \CT(x) = \omega(x - x_{target})^2 \quad\text{and}\quad \Ct(x,t) = \begin{cases}
									      10^4 & \text{if } t = \frac{T}{2} \text{ and } x \in Obstacle\\
                                                                              0    & \text{else}
                                                                             \end{cases}~,
\end{equation}
where $\omega$ is a weight. We considered a discretisation with time step $0.02s$, i.e. 100 time steps.

We compare the true optimal policy to those obtained using two variants of the proposed estimator, $\bPsi_{\text{OC}}$ and $\bPsi_{\text{RL}}$. The latter is based on a Reinforcement learning setting, learning from trajectory data without access to the cost, and uses the approach for sample sharing across time steps discussed in Section~\ref{sec:dyn:sample_reuse}. Meanwhile, $\bPsi_{\text{OC}}$ is based on single transitions from uniformly sampled start states and uses knowledge of the cost function to evaluate $\expC$ in each step. In both cases we use the low rank approximation (see supplementary material) and square exponential kernels $\psi(x,y) = \exp\{(x-y)^2/\lambda\}$ with $\lambda$ set to the median distance of the data. For comparison, we also consider two alternative approaches -- firstly, the trajectory based Monte Carlo approach of \cite{Theodorou:PI:RBDyn}, using the same number of trajectories as used in the Reinforcement Learning setting and on the other hand, a variational approximation, specifically a Laplace approximation to the true $\Psi$ to obtain a linear approximation of the optimal policy. As can be seen in Fig.~\ref{fig:ds}(b), the proposed approach leads to policies which significantly improve upon those based on the alternative Monte Carlo approach and which are comparable to those obtained from the variational approximation, which however was computed based on knowledge of the true $\Psi$. In particular, note that the proposed approach makes better use of the sample provided, finding a policy which is applicable for varying starting positions, as illustrated in Fig.~\ref{fig:ds}(a). As seen from the trajectories in Fig.~\ref{fig:ds}(a), the Monte Carlo approach on the other hand fails to capture the multi modality of the optimal policy leading to severely impoverished results when applied to starting point B without sampling a new data set (cf. Fig.~\ref{fig:ds}(b)). The variational approximation on the other hand similarly requires re-computation for each new starting location, without which results would also be significantly affected.

To illustrate the dependence of the estimate on the sample size we compare in Fig.~\ref{fig:ds}(c) the evolution of the $L_1$ error of the estimates of $\Psi$ at time $t=0$. Sample size refers to total number of transition samples seen, hence for $\bPsi_{\text{RL}}$ the number of trajectories was the sample size divided by 100. In order to also highlight the advantages of the sample re-use afforded by the approach in Section~\ref{sec:dyn:sample_reuse}, we also compare with $\bPsi$, the basic estimator given data of the same form as $\bPsi_{\text{RL}}$, i.e. recursive application of \refeq{eq:pi:rkhs:emp} without sample sharing across time steps.

\begin{figure}
\centering
\captionsetup[subfloat]{nearskip=0pt, captionskip=0pt}
\subfloat[]{
\label{fig:ds:setup}
%
%
\tikzset{
every pin/.style={font=\normalsize},
small dot/.style={thick, fill=black,circle,scale=0.3}
}
\definecolor{fillColorA}{rgb}{1,0.0,0.0}
\definecolor{edgeColorA}{rgb}{0.7,0.0,0.0}
\definecolor{fillColorB}{rgb}{0.2,0.5,1}
\definecolor{edgeColorB}{rgb}{0.0,0.0,0.7}
\definecolor{fillColorC}{rgb}{0.2,0.5,0.8}
\definecolor{edgeColorC}{rgb}{0.1,0.1,0.8}
\begin{tikzpicture}[scale=0.5,
        baseline]
\begin{pgfinterruptboundingbox}
\begin{axis}[%
clip=false,
font={\LARGE},
width =0.6\textwidth,
height=0.6\textwidth,
ylabel={Position},
xlabel={Time},
xtick={0,25,50,75,100},
xticklabels={0,0.5,1,1.5,2},
xmin=0, xmax=100,
ymin=-4, ymax=4,
axis on top,
every axis y label/.style=
{at={(-0.1,0.5)},rotate=90,anchor=center},
legend entries={$\bPsi_{\text{RL}}$,MC},
legend style = {font = \normalsize}
]

\addplot [
color=edgeColorA,
solid,
ultra thick
]
coordinates{ (1,-3) (2,-2.92277) (3,-2.94866) (4,-2.87961) (5,-2.83825) (6,-2.84573) (7,-2.85604) (8,-2.8368) (9,-2.85992) (10,-2.87824) (11,-2.8768) (12,-2.90408) (13,-2.90467) (14,-2.87869) (15,-2.87957) (16,-2.88875) (17,-2.88831) (18,-2.87425) (19,-2.86468) (20,-2.82525) (21,-2.82762) (22,-2.86955) (23,-2.88316) (24,-2.87255) (25,-2.88585) (26,-2.88359) (27,-2.91962) (28,-2.89729) (29,-2.93023) (30,-2.93667) (31,-2.93637) (32,-2.91843) (33,-2.88579) (34,-2.91731) (35,-2.90147) (36,-2.90407) (37,-2.88035) (38,-2.88675) (39,-2.88389) (40,-2.91266) (41,-2.88031) (42,-2.91314) (43,-2.92364) (44,-2.91337) (45,-2.81946) (46,-2.83056) (47,-2.80119) (48,-2.85669) (49,-2.87733) (50,-2.86116) (51,-2.88795) (52,-2.86993) (53,-2.81634) (54,-2.7186) (55,-2.64296) (56,-2.55874) (57,-2.5507) (58,-2.5161) (59,-2.49413) (60,-2.46935) (61,-2.40702) (62,-2.38299) (63,-2.31089) (64,-2.27418) (65,-2.23948) (66,-2.17246) (67,-2.14874) (68,-2.11126) (69,-2.07804) (70,-2.07056) (71,-2.00839) (72,-1.9642) (73,-1.96147) (74,-1.92347) (75,-1.86576) (76,-1.80348) (77,-1.78014) (78,-1.80995) (79,-1.82893) (80,-1.7517) (81,-1.72232) (82,-1.65327) (83,-1.59932) (84,-1.53832) (85,-1.45851) (86,-1.43807) (87,-1.38456) (88,-1.31697) (89,-1.24026) (90,-1.19908) (91,-1.11487) (92,-1.05423) (93,-1.00665) (94,-0.909492) (95,-0.820886) (96,-0.768168) (97,-0.715288) (98,-0.605618) (99,-0.511918) (100,-0.417482)
};

\addplot [
color=edgeColorA,
dashed,
ultra thick
]
coordinates{ (1,-3) (2,-2.98524) (3,-3.23359) (4,-3.27825) (5,-3.38261) (6,-3.46481) (7,-3.6758) (8,-3.81335) (9,-3.98953) (10,-3.70528) (11,-3.66922) (12,-3.78241) (13,-3.3282) (14,-3.4726) (15,-3.3327) (16,-3.34232) (17,-3.15154) (18,-3.00919) (19,-2.92544) (20,-2.6452) (21,-2.64932) (22,-2.47039) (23,-2.34957) (24,-2.35077) (25,-2.51355) (26,-2.45425) (27,-2.87486) (28,-3.14535) (29,-3.26951) (30,-3.26046) (31,-2.9686) (32,-3.12572) (33,-2.86528) (34,-2.97876) (35,-2.7157) (36,-2.89426) (37,-3.06828) (38,-2.91108) (39,-2.81262) (40,-3.04087) (41,-3.22238) (42,-3.11952) (43,-2.84993) (44,-2.36878) (45,-2.13874) (46,-2.45265) (47,-2.73467) (48,-2.59361) (49,-2.63966) (50,-2.65714) (51,-2.40662) (52,-2.55871) (53,-2.05961) (54,-1.5784) (55,-1.29586) (56,-1.17096) (57,-1.20535) (58,-1.10356) (59,-1.11348) (60,-1.11001) (61,-1.06565) (62,-1.09113) (63,-1.05717) (64,-0.946051) (65,-0.908659) (66,-0.908302) (67,-0.841076) (68,-0.73446) (69,-0.599972) (70,-0.594116) (71,-0.528759) (72,-0.514812) (73,-0.579912) (74,-0.483102) (75,-0.524454) (76,-0.466569) (77,-0.424819) (78,-0.370963) (79,-0.295084) (80,-0.270492) (81,-0.300521) (82,-0.325082) (83,-0.32396) (84,-0.300778) (85,-0.209124) (86,-0.224388) (87,-0.248532) (88,-0.262581) (89,-0.259776) (90,-0.272686) (91,-0.315968) (92,-0.332693) (93,-0.272234) (94,-0.251997) (95,-0.218242) (96,-0.246291) (97,-0.201994) (98,-0.0509055) (99,-0.065026) (100,-0.0902452)
};

\addplot [
color=gray,
solid,
line width=4.0pt
]
coordinates{ (50.5,4) (50.5,3.3)
};

\addplot [
color=gray,
solid,
line width=4.0pt
]
coordinates{ (50.5,2.5) (50.5,-2.5)
};

\addplot [
color=gray,
solid,
line width=4.0pt
]
coordinates{ (50.5,-3.3) (50.5,-4)
};

\node[small dot, pin={[pin edge={edgeColorA, ultra thick}]left:{\LARGE \textbf{\color{edgeColorA} A}}}] at (axis cs: 0, -3) {};
\addplot [
solid,
mark = *,
mark options= {%
  scale=2.5,fill=fillColorA,draw=edgeColorA,line width = 2pt 
},
]
coordinates{ (0,-3) 
};

\node[small dot, pin={[pin edge={edgeColorB, ultra thick}]above right:{\LARGE \textbf{\color{edgeColorB} B}}}] at (axis cs: 0, 1.75) {};
\addplot [
solid,
mark = *,
mark options= {%
  scale=2.5, fill=fillColorB,draw=edgeColorB, line width = 2pt
},
]
coordinates{ (0,1.75) 
};

\addplot [
solid,
mark = *,
mark options= {%
  scale=2.5, fill=gray,draw=gray!70!black, line width = 2pt
},
]
coordinates{ (100,0) 
};

\addplot [
color=edgeColorA,
solid,
ultra thick
]
coordinates{ (1,-3) (2,-2.92277) (3,-2.94866) (4,-2.87961) (5,-2.83825) (6,-2.84573) (7,-2.85604) (8,-2.8368) (9,-2.85992) (10,-2.87824) (11,-2.8768) (12,-2.90408) (13,-2.90467) (14,-2.87869) (15,-2.87957) (16,-2.88875) (17,-2.88831) (18,-2.87425) (19,-2.86468) (20,-2.82525) (21,-2.82762) (22,-2.86955) (23,-2.88316) (24,-2.87255) (25,-2.88585) (26,-2.88359) (27,-2.91962) (28,-2.89729) (29,-2.93023) (30,-2.93667) (31,-2.93637) (32,-2.91843) (33,-2.88579) (34,-2.91731) (35,-2.90147) (36,-2.90407) (37,-2.88035) (38,-2.88675) (39,-2.88389) (40,-2.91266) (41,-2.88031) (42,-2.91314) (43,-2.92364) (44,-2.91337) (45,-2.81946) (46,-2.83056) (47,-2.80119) (48,-2.85669) (49,-2.87733) (50,-2.86116) (51,-2.88795) (52,-2.86993) (53,-2.81634) (54,-2.7186) (55,-2.64296) (56,-2.55874) (57,-2.5507) (58,-2.5161) (59,-2.49413) (60,-2.46935) (61,-2.40702) (62,-2.38299) (63,-2.31089) (64,-2.27418) (65,-2.23948) (66,-2.17246) (67,-2.14874) (68,-2.11126) (69,-2.07804) (70,-2.07056) (71,-2.00839) (72,-1.9642) (73,-1.96147) (74,-1.92347) (75,-1.86576) (76,-1.80348) (77,-1.78014) (78,-1.80995) (79,-1.82893) (80,-1.7517) (81,-1.72232) (82,-1.65327) (83,-1.59932) (84,-1.53832) (85,-1.45851) (86,-1.43807) (87,-1.38456) (88,-1.31697) (89,-1.24026) (90,-1.19908) (91,-1.11487) (92,-1.05423) (93,-1.00665) (94,-0.909492) (95,-0.820886) (96,-0.768168) (97,-0.715288) (98,-0.605618) (99,-0.511918) (100,-0.417482)
};

\addplot [
color=edgeColorB,
solid,
ultra thick
]
coordinates{ (1,1.75) (2,1.76207) (3,1.73325) (4,1.78323) (5,1.76661) (6,1.79046) (7,1.82502) (8,1.83108) (9,1.82811) (10,1.83872) (11,1.81587) (12,1.82585) (13,1.79688) (14,1.80904) (15,1.849) (16,1.83633) (17,1.81149) (18,1.81401) (19,1.82053) (20,1.82742) (21,1.85015) (22,1.8398) (23,1.84505) (24,1.92649) (25,1.9589) (26,1.9729) (27,2.00518) (28,2.04968) (29,2.06985) (30,2.07507) (31,2.08322) (32,2.11483) (33,2.16707) (34,2.24769) (35,2.26903) (36,2.27213) (37,2.30529) (38,2.34473) (39,2.38732) (40,2.40522) (41,2.41486) (42,2.41463) (43,2.40127) (44,2.39408) (45,2.36104) (46,2.42964) (47,2.45152) (48,2.46186) (49,2.4932) (50,2.5211) (51,2.49779) (52,2.56403) (53,2.52005) (54,2.45985) (55,2.42538) (56,2.38115) (57,2.37865) (58,2.31668) (59,2.27116) (60,2.22969) (61,2.2513) (62,2.25022) (63,2.24075) (64,2.17629) (65,2.12319) (66,2.05866) (67,1.98687) (68,1.93764) (69,1.96463) (70,1.88007) (71,1.84715) (72,1.83062) (73,1.81582) (74,1.75734) (75,1.71767) (76,1.7098) (77,1.68112) (78,1.66155) (79,1.61264) (80,1.57562) (81,1.54873) (82,1.4992) (83,1.4608) (84,1.44903) (85,1.38673) (86,1.28964) (87,1.2319) (88,1.19798) (89,1.18676) (90,1.1146) (91,1.09112) (92,1.05757) (93,1.03283) (94,0.970808) (95,0.937129) (96,0.901501) (97,0.819581) (98,0.709426) (99,0.617487) (100,0.463461)
};

\addplot [
color=edgeColorB,
dashed,
ultra thick
]
coordinates{ (1,1.75) (2,1.21079) (3,0.673113) (4,0.161495) (5,-0.329521) (6,-0.783091) (7,-1.3785) (8,-1.8645) (9,-2.33032) (10,-2.74072) (11,-3.17202) (12,-3.53183) (13,-3.41851) (14,-3.38063) (15,-3.318) (16,-3.38202) (17,-3.1863) (18,-3.06878) (19,-2.94471) (20,-2.60029) (21,-2.6627) (22,-2.49264) (23,-2.37374) (24,-2.34291) (25,-2.51414) (26,-2.54275) (27,-2.90551) (28,-3.17129) (29,-3.2648) (30,-3.26656) (31,-2.92712) (32,-3.14365) (33,-2.84981) (34,-2.953) (35,-2.70863) (36,-2.89317) (37,-2.99051) (38,-2.91096) (39,-2.81133) (40,-3.08277) (41,-3.1934) (42,-3.09474) (43,-2.8507) (44,-2.39764) (45,-2.11787) (46,-2.45563) (47,-2.75153) (48,-2.53492) (49,-2.59316) (50,-2.68911) (51,-2.40859) (52,-2.56969) (53,-2.12591) (54,-1.62674) (55,-1.30569) (56,-1.19971) (57,-1.15588) (58,-1.12606) (59,-1.10154) (60,-1.0792) (61,-1.01035) (62,-1.11241) (63,-1.05447) (64,-0.977403) (65,-0.946303) (66,-0.920614) (67,-0.858935) (68,-0.732676) (69,-0.59925) (70,-0.60251) (71,-0.544723) (72,-0.530067) (73,-0.553575) (74,-0.507482) (75,-0.480113) (76,-0.447047) (77,-0.408125) (78,-0.34425) (79,-0.32217) (80,-0.270148) (81,-0.356396) (82,-0.311691) (83,-0.35174) (84,-0.263946) (85,-0.274387) (86,-0.276531) (87,-0.169807) (88,-0.296116) (89,-0.231527) (90,-0.301309) (91,-0.334601) (92,-0.321999) (93,-0.240856) (94,-0.256085) (95,-0.253279) (96,-0.164544) (97,-0.194148) (98,-0.0926771) (99,-0.0338818) (100,-0.00912387)
};
\end{axis}
\end{pgfinterruptboundingbox}
\useasboundingbox
(current axis.below south west)
rectangle (current axis.above north east);
\end{tikzpicture}
}
\subfloat[]{
\label{fig:ds:Jpi}
\tikzset{
every pin/.style={font=\normalsize},
small dot/.style={thick, fill=black,circle,scale=0.3}
}

\begin{tikzpicture}[scale=0.5,baseline]

\definecolor{fillColorA}{rgb}{1,0.0,0.0}
\definecolor{edgeColorA}{rgb}{0.7,0.0,0.0}
\definecolor{fillColorB}{rgb}{0.2,0.5,1}
\definecolor{edgeColorB}{rgb}{0.0,0.0,0.7}
\definecolor{fillColorC}{rgb}{0.2,0.5,0.8}
\definecolor{edgeColorC}{rgb}{0.1,0.1,0.8}
\begin{axis}[%
scaled y ticks = base 10:-3,
font={\LARGE},
width =0.45\textwidth,
height=0.6\textwidth,
xmin=-0.5, xmax=4.5,
ymin=0.0, 
ylabel={Expected Cost},
xlabel={\ },
xtick={0,1,2,3,4},
xticklabels={{$\optpi$},{MC},{Var},{$\bPsi_{\text{\small RL}}$},{$\bPsi_{\text{\small OC}}$}},
axis on top,
every axis y label/.style=
{at={(-0.12,0.5)},rotate=90,anchor=center}
]

\addplot[
fill=fillColorA,
draw=edgeColorA,
ybar,
thick,
bar width=10pt,
]
coordinates{
 (-0.25,119.7009) (0.75,1725.35) +- (0,197.728)(1.75,273.3345)+-(0,0)(2.75,558.406) +- (0,229.663)(3.75,486.253) +- (0,103.704)%
};

\addplot[
only marks
]
plot[error bars/.cd, y dir = both, y explicit]
coordinates{
 (0.75,1725.35) +- (0,197.728)(1.75,273.3345)+-(0,386.5533)(2.75,558.406) +- (0,229.663)(3.75,486.253) +- (0,103.704)%
};

%
\addplot[
 fill=fillColorB,
 draw=edgeColorB,
 ybar,
 thick,
 bar width=10pt,
 ]
 coordinates{
  (0.25,116.7097)(1.25,7346.97) +- (0,207.856)(2.25,276.3513)+-(0,0)(3.25,816.746) +- (0,200.411)(4.25,574.943) +- (0,95.0806)
};
\addplot[only marks]
 plot[error bars/.cd, y dir = both, y explicit]
 coordinates{
  (1.25,7346.97) +- (0,207.856)(2.25,276.3513)+-(0,390.8198)(3.25,816.746) +- (0,200.411)(4.25,574.943) +- (0,95.0806)
};
\end{axis}
\end{tikzpicture}
}
\subfloat[]{
\label{fig:ds:psi:ex}
\includegraphics[height=0.245\textwidth]{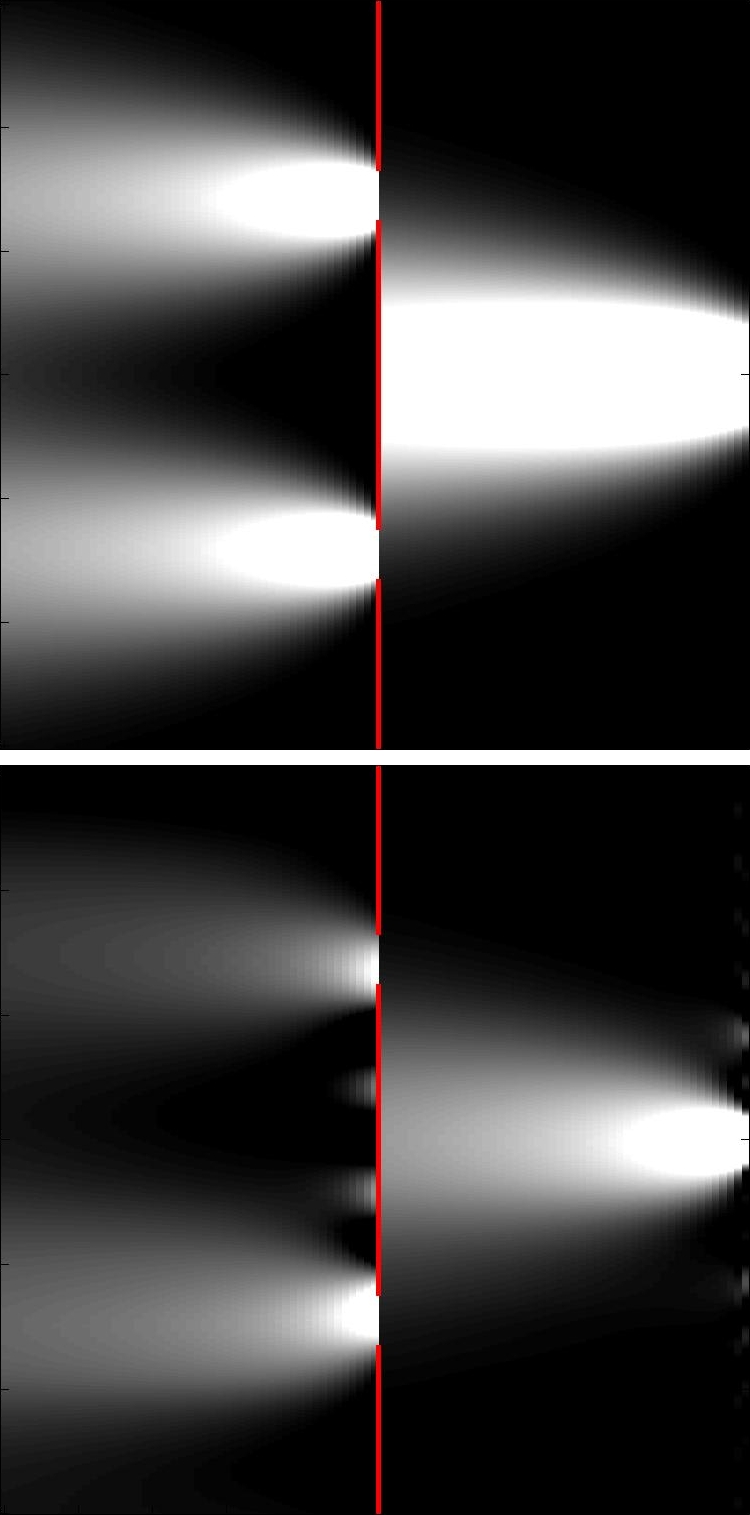}
}
\subfloat[]{
\label{fig:ds:psi}
\begin{tikzpicture}[scale=0.5, baseline]
\begin{axis}[%
font={\LARGE},
xtick={10000,20000},
scaled x ticks=false,
legend style={
cells={anchor=west},
},
xticklabels={$10^4$,$2\cdot10^4$},
minor tick num=1,
xlabel={Transition Samples},
ylabel={$\norm[L_1]{\bPsi_0 - \Psi(\cdot,0)}$},
width = 0.5\textwidth,
height= 0.6\textwidth,
xmin=0, xmax=25000,
ymin=0, ymax=4.99,
axis on top,
every axis y label/.style=
{at={(-0.15,0.5)},rotate=90,anchor=center}
]

\addplot [
color=gray,
solid,
ultra thick,
mark=*
]
plot [error bars/.cd, y dir = both, y explicit]
coordinates{ (1000,1.54366) +- (0,0.539748) (2000,1.41324) +- (0,0.695359) (4000,0.906453) +- (0,0.226353) (6000,0.778993) +- (0,0.142454) (8000,0.757037) +- (0,0.122371) (10000,0.772976) +- (0,0.138471) (20000,0.658886) +- (0,0.0593721) (40000,0.625135) +- (0,0.0270534) (60000,0.603079) +- (0,0.0239207)
};
\addlegendentry{$\bPsi$}

\addplot [
color=blue,
solid,
ultra thick,
mark=*
]
plot [error bars/.cd, y dir = both, y explicit]
coordinates{ (1000,4.02803) +- (0,7.29773) (2000,1.8587) +- (0,2.8879) (4000,0.416072) +- (0,0.198785) (6000,0.344235) +- (0,0.204038) (8000,0.30137) +- (0,0.177084) (10000,0.289332) +- (0,0.141028) (20000,0.160278) +- (0,0.0657207) (40000,0.0953596) +- (0,0.0262281) (60000,0.0793779) +- (0,0.0230265)
};
\addlegendentry{$\bPsi_{\text{\small RL}}$}

\addplot [
color=red,
solid,
ultra thick,
mark=*
]
plot [error bars/.cd, y dir = both, y explicit]
coordinates{ (1000,2.57602) +- (0,2.63927) (2000,1.2291) +- (0,1.01941) (4000,0.493858) +- (0,0.546269) (6000,0.40459) +- (0,0.563312) (8000,0.214606) +- (0,0.129799) (10000,0.173871) +- (0,0.0838607) (20000,0.153795) +- (0,0.0714842) (40000,0.137853) +- (0,0.0666168) (60000,0.114274) +- (0,0.0355282)
};
\addlegendentry{$\bPsi_{\text{\small OC}}$}

\end{axis}
\end{tikzpicture}
}
\caption{\label{fig:ds}Results for the double slit problem.
\subref{fig:ds:setup} Problem setup and mean trajectories from policies MC and $\bPsi_{\text{RL}}$ for two start points are shown. Obstacles and target are shown in gray. 
\subref{fig:ds:Jpi} Empirical expected cost for policies based on various methods for the two start states.
\subref{fig:ds:psi:ex} The true $\Psi$ \emph{(top)} and the estimate $\bPsi_{\text{OC}}$ \emph{(bottom)} based on $10^4$ samples.  
\subref{fig:ds:psi} The $L_1$ error of estimates of $\Psi(\cdot,0)$ as a function of (transition) sample size, n.b. in case of $\bPsi$ and $\bPsi_{\text{RL}}$ data was sampled as 100 step trajectories, for various estimators.}
\end{figure}

\subsection{Arm Subspace Reaching Task}
We consider reaching tasks on a subspace of the end-effector space of a torque controlled 5dof arm, simulating constrained tasks such as, for e.g., drawing on a whiteboard or pushing objects around on a table. Here the skill component consists of moving with the end-effector staying close to a two dimensional task space, while the task instances are given by specific reach targets. The task space used is a linear subspace of the end effector space, n.b., hence, a non linear subspace of the joint space, and the cost comprises the two components
\begin{equation}
      \Cskill(\x,t) = \omega_{skill}\norm{\mat{J}\varphi(\x) - \vec{j}}^2 \quad\text{and}\quad \Ctask(\x,\theta) = \omega_{task}\norm{\varphi(\x) - \theta}^2~,
\end{equation}
where $\varphi(\cdot)$ is the mapping from joint to end-effector coordinates, $\mat{J}$ \& $\vec{j}$ define the task subspace, $\theta$ specifies the reaching target and $\omega$'s are weights. We again consider position control over a 2s horizon with a 0.02s discretisation.

This task is challenging for sample based approaches as the low cost trajectories are restricted to a small subspace, necessitating large sample sizes to obtain good results for an individual reaching target, even if, as suggested in \cite{Theodorou:PI:RBDyn} and done here, an inverse dynamics policy is used which significantly improves end-effector exploration. However, concentrating on the case of changing targets, we exploit the ideas from Section~\ref{sec:abstraction:task} by assuming the operators have been estimated under the skill augmented dynamics\footnote{n.b., while here such a sample is generated explicitly, the more time consuming approach of using the importance sample based estimator and collecting a sample under $\Xz$ could be used} (cf. \refeq{eq:pi:task}), and consider subsequent learning for a novel task using the estimator from Section~\ref{sec:dyn:sample_reuse}, utilising the already estimated operators in two ways. On the one hand, they are directly used in the calculation of $\bPsi$, on the other hand, noting, that as the trajectories are only required to provide $\set{D}'$, hence do not have to be sampled under a specific policy, we use the policy arising when considering $\Cskill$ only, i.e., the skill policy associated with $\bPsi$ computed using the given operators and $\Ctask(\cdot) = 0$.

\begin{figure}
\centering
\captionsetup[subfloat]{position=top}
\subfloat[]{
\label{fig:arm:traintraj}
\includegraphics[width=0.29\textwidth]{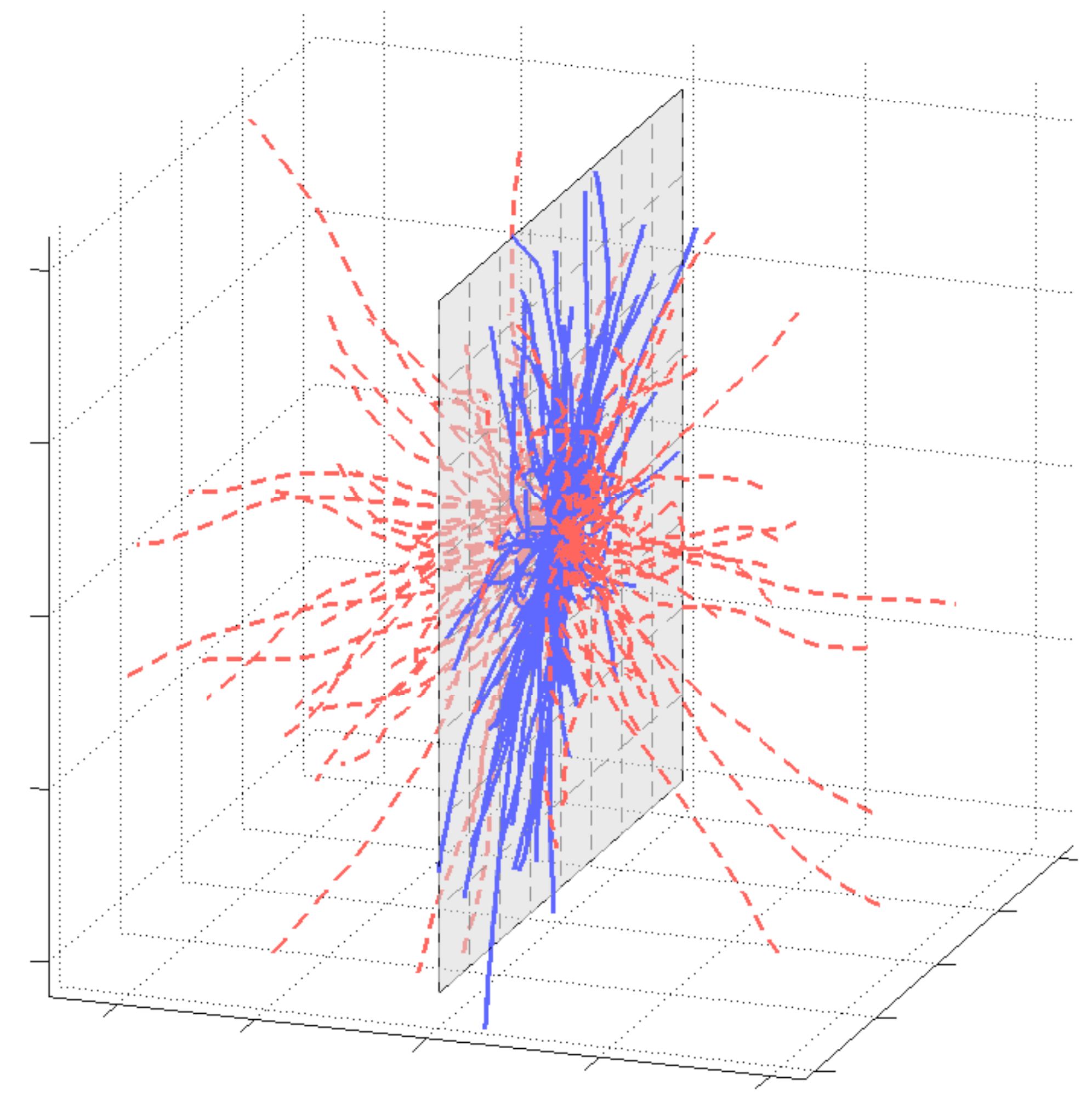}
}
\subfloat[]{
\label{fig:arm:setup}
\includegraphics[width=0.29\textwidth]{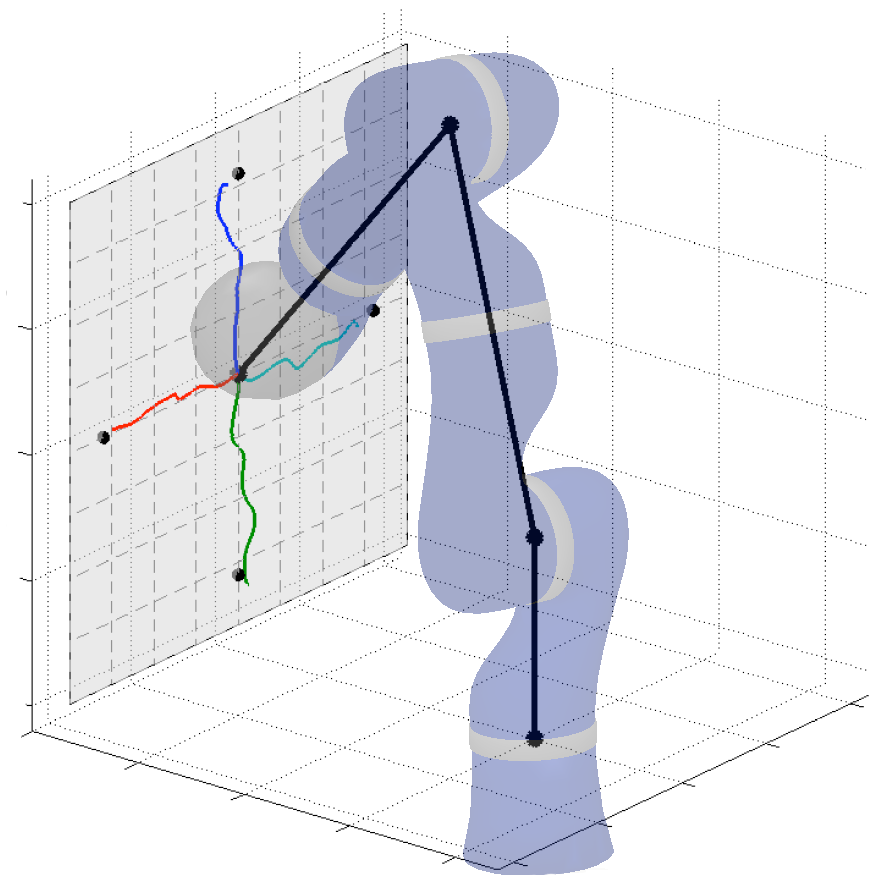}
}
\subfloat[]{
\label{fig:arm:psi}
\begin{tikzpicture}[scale=0.6,baseline]
\begin{axis}[%
font={\LARGE},
xlabel={Number Trajectories},
ylabel={$\norm[L_1]{\bPsi_0 - \bPsi^{ref}_0}$},
ytick={0,1},
minor y tick num=4,
width = 0.6\textwidth,
height= 0.6\textwidth,
xmin=0, xmax=220,
ymin=-0.19, ymax=1.2,
axis on top,
]

\addplot [
color=blue,
solid,
mark=*,
ultra thick
]
plot [error bars/.cd, y dir = both, y explicit]
coordinates{ (0,1.00439) +- (0,0.000136296) (10,0.775183) +- (0,0.395216) (50,0.371416) +- (0,0.3925) (100,0.254292) +- (0,0.30497) (150,0.23746) +- (0,0.288876) (200,0.252735) +- (0,0.289015)
};
\end{axis}
\end{tikzpicture}
}
\caption{\label{fig:arm}Results in the reaching task.
\subref{fig:arm:traintraj} Training trajectories under the skill augumented policy (solid blue) and $\nullpi$ (dashed red) with task space.
\subref{fig:arm:setup} Illustration of the task setup and example trajectories of policies after 100 training trajectories for a set of reaching tasks. The black dots show individual reaching targets with the arm shown in it's initial pose.
\subref{fig:arm:psi} The $L_1$ error of estimates of $\Psi(\cdot,0)$ as a function of training trajectories measured with respect to an estimate trained on 5000 trajectories. The data point coresponding to \#traj = 0 is based on the estimate is of $\bPsi$ taking only $\Cskill$ into account (see text for details).}
\end{figure}

The advantage of sampling under the skill policy is illustrated in Fig.~\ref{fig:arm}(a) where sample trajectories under both the skill and null policy are shown, demonstrating that the former more effectively explores the the task relevant sub space. Mean trajectories for policies learned from 100 trajectories for a set of tasks are illustrated in Fig.~\ref{fig:arm}(b). In Fig.~\ref{fig:arm}(c) we plot the $L_1$ error of $\bPsi$ as a function of trajectories averaged over ten $\theta$. As the true $\Psi$ is not available for this task we show the error w.r.t. a $\bPsi$ computed from 5000 trajectories, principally to illustrate the rapid convergence of the estimator.

\section{Conclusion}
We have presented a novel approach for solving stochastic optimal control problems which are of the path integral control form
using Monte Carlo estimates of operators arising from a RKHS embedding of the problem, leading to a consistent estimate of $\Psi$. While direct application of Monte Carlo estimation to point evaluation of $\Psi$ also yields a consistent estimate, it is impractical for computation of controls for anything but simple problems, requiring a trajectory sample for each state at which an action is to be computed. Although previous work, e.g., \cite{Theodorou:PI:RBDyn} and similarly \cite{Theodorou:PIPI}, has suggested approaches to overcome the problem of sample complexity, these sacrifice consistency in the process and we demonstrate that the proposed approach significantly improves upon them in terms of generalization to a policy (cf. results in Fig.\ref{fig:ds}(a,b)). We furthermore show that the presented estimators allow for sample re-use in situations which previously required an entirely novel sample set. In particular we consider transfer in cases where execution of several, potentially related, tasks on the same plant is required, demonstrating that it is possible to exploit samples from all tasks to learn invariant aspects.


Note that as $\expC$ itself defines a local optimal control problem, an alternative perspective on the proposed method is as a principled approach to combining solutions to local control problem to solve a more complex large scale problem. In future work we aim to elaborate on this interpretation by combining the methods persented with alternative approaches, e.g., variational methods, which may be sufficient to provide good estimates for the comparatively simpler local problems.

The choice of kernel has been largely ignored here, but one may expect improved results by making informed kernel choices based on prior knowledge about the structure of the problem.

\newpage
\bibliography{references}
\newpage
\appendix
\section{Alternative embedding}
As indicated in the main text (cf. section \ref{sec:emb:analytical}) an alternative representation to the embedding \refeq{eq:pi:rkhs:analytical:b} exists. Observe that for the purposes of the expectation the conditioning variable is fixed and $\expC$ is in fact only a function of it's second argument, making it possible to apply \refeq{eq:op:U} embedding  $\bX_{i+1}|x_i$ into the tensor space in which the product of $\Psi$ and the partially evaluated $\Phi$ resides. Formally define the operator for partial evaluation on $\HPhi$
\begin{equation}
\op{R}_{x}\left[h\right] = h(x,\cdot) \quad \forall x\in\Real^{D_\x}, h\in\HPhi
\end{equation}
In particular note that for $\op{R}_{x}:\HPhi\rightarrow\rkhs{H}{\phi_{\x}}$ where $\phi_{\x} = \phi((x,\cdot),(x,\cdot))$. We can now write $\expC_i(\x,\cdot) = \op{R}_{\x}\left[\expC_i\right] \in \rkhs{H}{\kphi_\x}$ and application of \refeq{eq:op:E} and \refeq{eq:op:U} to \refeq{eq:pi:finite} leads to
\begin{align}
\Psi_i(\x)
& = \Exp{\bX_{i+1}|\bX_{i}=\x}{\OpR{\x}{\expC_i}(\bX_{i+1})\cdot\Psi_{i+1}(\bX_{i+1})} \\
& = \left<\OpR{\x}{\expC_i}\otimes\Psi_{i+1}, \OpE{\kc}{\bX_{i+1}|\bX_{i} = \x}\right> \\
\label{eq:pi:rkhs:analytical:a}
& = \left<\OpR{\x}{\expC_i}\otimes\Psi_{i+1}, \opU{\kc\ka}\circ\OpE{\ka}{\x}\right>
\end{align}
where $\Hc = \rkhs{H}{\kphi_\x}\otimes\HPsi$ and $\ka$ is some kernel of our choosing on $\Real^{D_x}$ which again we take to be $\psi$.

Although \refeq{eq:pi:rkhs:analytical:a} is formally equivalent to the embedding derived in the main text, i.e., \refeq{eq:pi:rkhs:analytical:b}, the yield different empirical estimates. Specifically, applying \refeq{eq:op:U:emp} to the \refeq{eq:pi:rkhs:analytical:a} we obtain $\hat{\Psi}_{i}(\x) = \mat{G}^{\kpsi}_{\x\set{X}}\alpha(\x)$ with
\begin{equation}
	\alpha(\x) = \trans{\left[\mat{G}^{\kphi}_{(\x\set{Y})\set{R}}\beta \odot \Gram{\kpsi'}{YX'}\alpha'\right]}\inv{(\Gram{\kpsi}{XX} + \epsilon n \mat{I})}
\end{equation}
Hence, although this approach allows us to evaluate $\hat{\Psi}_i$ at specific points, we do not directly obtain a  a finite dimensional representation of $\hat{\Psi}_i$ in some RKHS. Furthermore, due to the dependence on the evaluation point, the Gram matrix $\mat{G}^{\kphi}_{(\x\set{X})\set{R}}$ can in general not be pre computed. None the less this form may have it's applications for a forward, backwards algorithm where $\mat{G}^{\kphi}_{(\x\set{X})\set{R}}$ is used for selection of an active set $\set{X}$ for which $\alpha$'s are computed in a backwards pass.

\section{Alternative Estimators}
We now discuss the two alternative estimators based on weighted samples alluded to in the main text (cf. section \ref{sec:estimators}).
\subsection{Low rank Approximation}
First, we address the computational complexity of \refeq{eq:pi:rkhs:emp}, which is $\fun{O}(m^3)$ for the matrix inversion, which may be precomputed, and $O(m^2)$ for subsequent computations. Although such costs are acceptable for reasonably sized problems, they may prove prohibitive for application to realistic robotic systems. However we can apply a Gram-Schmidt orthogonalisation of $\gram{\ka}{X}$, $\gram{\ka}{X'}$, as proposed by \cite{Gretton:KBP}. Summarising we approximate $\gram{\ka}{X} \approx \gram{\ka}{Y}\W_x$ and  $\gram{\ka}{X'} \approx \gram{\ka}{\hat{X}'}\W_{x'}$ , where $\set{Y} \subseteq \set{X}$, $\set{Y'} \subseteq \set{X'}$ and $\W_x,\W_{x'}$ are weight matrices. Substituting into \refeq{eq:pi:rkhs:emp} we may then obtain the alternative estimator
\begin{equation}
\label{eq:pi:rkhs:emp:w:low}
  \alpha_i = \trans{\left[\Gram{\kphi}{D'B}\beta \odot \Gram{\kpsi}{Y'A}\alpha_{i+1}\right]} \W_{x'}\trans{\W_x} \inv{\left(\W_x\trans{\W_x} + \epsilon m \inv{{\Gram{\psi}{YY}}}\right)}\inv{{\Gram{\ka}{YY}}}
\end{equation}
This is computationally advantageous as with $|\set{Y}| = \hat{m} \ll m$ the complexity reduces to $\fun{O}(\hat{m}^3 + \hat{m}^2m)$ and $\fun{O}(\hat{m}^2)$ for required pre-computations and per iteration respectively, often with minimal effects on the obtained results.

\subsection{Importance Sampling}
The estimator \refeq{eq:pi:rkhs:emp} is based on a sample from the distribution $p_{\nullpi}(X'|X)\mu(X)$ and while we are free to choose $\mu$, $p_{\nullpi}$ is specified by $X_{i+1}|X_i$, i.e. the uncontrolled dynamics. In practice it may be impractical to sample according to the uncontrolled dynamics, e.g., we may wish to improve the policy sequentially collecting new sample following the already learned, rather then the uninformed, policy. To address such situation we follow the importance sampling approach. Specifically note that
\begin{equation}
   \opC{\ka\kb}_{\Z\Y} = \Exp{(\Z',\Y')}{\frac{P(\Z',\Y')}{Q(\Z',\Y')}\ka(\Z,\cdot)\otimes\kb(\Y,\cdot)}~,
\end{equation}
where $P,Q$ are the p.d.f.s of the two joints $(\Z,\Y)$, $(\Z',\Y')$ and we assume $Q(z,y) = 0 \Rightarrow P(z,y) = 0$. Hence given a i.i.d. sample from $(\Z',\Y')$ and empirical estimate of
$\opC{\ka\kb}_{\Z\Y}$ is given by
\begin{equation}
\ophC{\ka\kb}{\set{D}} = \sum_{i=1}^m w_i \ka(\cdot, \z_i)\otimes\kb(\cdot,\y_i) \quad\text{, with}\quad w_i = P(z_i,y_i)/Q(z_i,y_i) ~.
\end{equation}
Applying these to \refeq{eq:op:U:asCov} to obtain an empirical estimate of $\opU{}$, it is easy to show that the based on a sample from $p_{\pi}(X'|X)\mu(X)$, formed from an alternative policy, the estimator $\bPsi_{i} = \gram{\kpsi}{X}\alpha_i$ with
\begin{equation}
\label{eq:pi:rkhs:emp:w:is}
    \alpha_{i} = \trans{\left[\Gram{\kphi}{DB}\beta \odot \Gram{\kpsi}{X'A}\alpha_{i+1}\right]}\W\inv{(\Gram{\ka}{XX} + \epsilon n \mat{I})}
\end{equation}
is obtained, where $\W$ is the diagonal weight matrix with $\W_{ii} = p_{\nullpi}(x'_i|x_i)/p_{\pi}(x'_i|x_i)$ and we again assume that $p_{\pi}(x'|x) = 0 \Rightarrow p_{\nullpi}(x'|x) = 0$.

\section{Proofs and Derivations}
\subsection{General Results for Path Integral Control}
\begin{theorem}
Let the optimal value function be bounded, say $\Copt(\cdot,t) < c$ then,
\begin{equation}
  \norm[\infty]{\bPsi(\cdot,t) - \Psi(\cdot,t)} \rightarrow 0 \Longrightarrow \norm[\infty]{\bar{\Copt}(\cdot,t) - \Copt(\cdot,t)} \rightarrow 0
\end{equation}
\end{theorem}
\begin{proof}
From \refeq{eq:Copt},
\begin{equation}
  \Copt(\cdot,t) < c \Rightarrow \Psi(\cdot,t) > c' > 0
\end{equation}
Now
\begin{align}
\norm[\infty]{\bar{\Copt}(\cdot,t) - \Copt(\cdot,t)}
= & \sup_\x |\bar{\Copt}(\x,t) - \Copt(\x,t)|\\
= & \lambda \sup_\x |\log\frac{\bPsi(\x,t)}{\Psi(\x,t)}|\\
= & \lambda \sup_\x |\log(\frac{\bPsi(\x,t) - \Psi(\x,t)}{\Psi(\x,t)} + 1)|\\
\leq & \lambda \sup_\x |\log(\frac{\bPsi(\x,t) - \Psi(\x,t)}{c'} + 1)|
\end{align}
and thus
\begin{equation}
  \norm[\infty]{\bPsi(\cdot,t) - \Psi(\cdot,t)} \rightarrow 0  \Rightarrow \frac{\bPsi(\x,t) - \Psi(\x,t)}{c'} + 1 \rightarrow 1 \Rightarrow \norm[\infty]{\bar{\Copt}(\cdot,t) - \Copt(\cdot,t)} \rightarrow 0
\end{equation}
\end{proof}

\subsection{Convergence of Estimates}
\begin{theorem}
Under the assumptions in the main text, the assumptions of lemma \ref{lem:U:conv} and assuming all relevant kernels satisfy $0 \leq \ka(x,x') \leq 1$, the estimator $\bPsi_i$ is consistent, i.e., $\norm[\set{H}]{\bPsi_i - \Psi_i}$ converges in probability.
\end{theorem}
\begin{proof}
\newcommand{\aPsi}{\ensuremath{\tilde{\Psi}}}
Let $\aPsi_i = \OphU{*}{}{\Phi\otimes\Psi_{i+1}}$ where $\ophU{*}{}$ is the adjoint of $\ophU{}{}$, i.e., $\aPsi_i$ captures the approximation arising due to the empirical embedding. Then using general relation $\norm[\set{H}]{\fun{T}h} \leq \norm[2]{\fun{T}}\norm[\set{H}]{h} \leq \norm[HS]{\fun{T}}\norm[\set{H}]{h}$ for bounds on operators, we have the bound
\begin{align}
\norm[\set{H}]{\aPsi_i - \Psi_i}
= &    \norm[\set{H}]{\OphU{*}{}{\Phi\otimes\Psi_{i+1}} - \OpU{*}{\Phi\otimes\Psi_{i+1}}} \\
\leq & \norm[\set{H}]{\Phi\otimes\Psi_{i+1}} \norm[HS]{\ophU{*}{} - \opU{*}} \\
= & \norm[\set{H}]{\Phi}\underbrace{\norm[\set{H}]{\Psi_{i+1}} \norm[HS]{\ophU{*}{} - \opU{*}}}_{=: \epsilon_i}
\end{align}
Now
\begin{align}
\norm[\set{H}]{\bPsi_i - \Psi_i}
\leq &  \norm[\set{H}]{\bPsi_i - \aPsi_i} + \norm[\set{H}]{\aPsi_i - \Psi_i} \\
\leq &  \norm[\set{H}]{\OphU{*}{}{\Phi\otimes\bPsi_{i+1}} - \OphU{*}{}{\Phi\otimes\Psi_{i+1}}} + \norm[\set{H}]{\Phi}\epsilon_i \\
\leq &  \norm[\set{H}]{\Phi\otimes\bPsi_{i+1} - \Phi\otimes\Psi_{i+1}}\norm[2]{\ophU{*}{}} + \norm[\set{H}]{\Phi}\epsilon_i \\
\leq &  \norm[\set{H}]{\Phi}\norm[\set{H}]{\bPsi_{i+1} - \Psi_{i+1}} + \norm[\set{H}]{\Phi}\epsilon_i
\end{align}
where in the last line we used $0 \leq \ka(x,x') \leq 1 \Rightarrow \norm[2]{\ophU{*}{}} \leq 1$. As we may further using lemma \ref{lem:U:conv} and the union bound construct $\epsilon$ s.t. with probaility $1-\delta$ simultaniously for all $\epsilon_i$, $\epsilon_i \leq \epsilon$. The result then follows by induction.
\end{proof}

\subsection{Auxillary Results}
\begin{lemma}[Song et al., 2010]
\label{lem:U:conv}
Assume the operator $\fun{C}_{YX}\fun{C}_{XX}^{-\frac{3}{2}}$ is Hilbert-Schmidt, then
\begin{equation}
  \norm[HS]{\hat{\fun{U}} - \fun{U}_{Y|X}} = \fun{O}(\lambda^{\frac{1}{2}} + \lambda^{-\frac{3}{2}}m^{-\frac{1}{2}})
\end{equation}
In particular if the regularization term $\lambda$ satisfies $\lambda \rightarrow 0$ and $m\lambda^3 \rightarrow \infty$, then $\norm[HS]{\ophU{\kb\ka}{D} - \opU{\kb\ka}}$ converges in probability.
\end{lemma}
\begin{proof}
    See \cite{Song:Tree} Theorem 1.
\end{proof}
\end{document}